\documentclass[journal]{ijcas}

\usepackage{times}
\usepackage{graphicx}
\usepackage{amssymb}
\usepackage{gensymb}
\usepackage{amsmath}
\usepackage{amsthm}
\usepackage{xcolor}
\usepackage{siunitx}
\usepackage{booktabs}
\usepackage{multirow}
\usepackage[normalem]{ulem}

\usepackage[ruled,vlined,linesnumbered]{algorithm2e}

\journalvolumn{VV}
\journalnumber{X}
\journalyear{YYYY}
\setarticlestartpagenumber{1}
\begin{document}

\title{
Aerial Chasing of a Dynamic Target in Complex Environments}

\begin{abstract}
Rapidly generating an optimal chasing motion of a drone to follow a dynamic target among obstacles is challenging due to numerical issues rising from multiple conflicting objectives and non-convex constraints.  
This study proposes to resolve the difficulties with a fast and reliable pipeline that incorporates 1) a target movement forecaster and 2) a chasing planner. They are based on a \textit{sample-and-check} approach that consists of the generation of high-quality candidate primitives and the feasibility tests with a light computation load.    
We forecast the movement of the target by selecting an optimal prediction among a set of candidates built from past observations.
Based on the prediction, we construct a set of prospective chasing trajectories which reduce the high-order derivatives, while maintaining the desired relative distance from the predicted target movement. Then, the candidate trajectories are tested on safety of the chaser and visibility toward the target without loose approximation of the constraints.
The proposed algorithm is thoroughly evaluated in challenging scenarios involving dynamic obstacles. Also,
the overall process from the target recognition to the chasing motion planning is implemented fully onboard on a drone, demonstrating real-world applicability. 
\end{abstract}


\maketitle

\runningtitle{2021}{Boseong Felipe Jeon, Changhyeon Kim, Hojoon Shin and 
        H. Jin Kim}{Aerial Chasing of a Dynamic Target in Complex Environments}{xxx}{xxxx}{x}

\makeAuthorInformation{
Boseong Felipe Jeon, Changhyeon Kim, Hojoon Shin and 
        *H. Jin Kim are with Aerospace Engineering, Seoul National University, South Korea (e-mails: \{junbs95,hyun91015,asdwer20\}@gmail.com and hjinkim@snu.ac.kr).
        
* Corresponding author.

This material is based upon work supported by the Ministry of Trade, Industry \& Energy(MOTIE, Korea) under Industrial Technology Innovation Program. No.10067206, 'Development of Disaster Response Robot System for Lifesaving and Supporting Fire Fighters at Complex Disaster Environment'
}

\theoremstyle{lemma}
\theoremstyle{proposition}

\theoremstyle{problem}
\newtheorem{problem}{Problem}

%
\newcommand{\pos}{\mathbf{x}}
\newcommand{\chaserPos}{\mathbf{x}_p}
\newcommand{\targetPosTrue}{\mathbf{x}_q}
\newcommand{\targetPos}{\hat{\mathbf{x}}_q}
\newcommand{\targetMean}{\mathbf{m}}
\newcommand{\targetVar}{Q}

\newcommand{\obstaclePos}{\mathbf{x}_r}
\newcommand{\obstaclePosJ}{\mathbf{x}_{r_j}}

\newcommand{\R}{\mathbb{R}}
\newcommand{\Rpos}{\mathbb{R}^3}
\newcommand{\polyBaseM}{\mathbf{e}^d_m(t)}
\newcommand{\polyBase}{\mathbf{c}}
\newcommand{\polyOrder}{m}
\newcommand{\polyCoeffChaser}{\mathbf{p}}
\newcommand{\polyCoeffTarget}{\mathbf{q}}
\newcommand{\polyCoeffObstacle}{\mathbf{r}}

\newcommand{\ViewPointSet}{\mathcal{U}}

\section{Introduction}
Absence of the exact prior about the future motion of the target of interest and the environment is one of the most challenging conditions in target-following tasks using a drone such as aerial videography and surveillance.    
The success of missions under the condition hinges on the reliability and efficiency \textcolor{black}{of the strategies to predict the target of interest and to generate the chasing motion of the drone based on the prediction.}
\label{sec:relwork}
\textcolor{black}{Among various studies related with chasing motion strategies,  
we focus on the recent works which constructed and examined real-time chasing planners for a drone in multiple cluttering obstacles.
Specifically, we highlight 
1) \textcolor{black}{non-convex-optimization-based receding horiozn planners (RHPs)}  \cite{nageli2017real,penin2018vision,bonatti2020autonomous} and 2)  corridor-based RHPs \cite{ jeon2002online,chen2016tracking,jeon2019online}. }  

The first group \cite{nageli2017real,penin2018vision,bonatti2020autonomous} generates the chasing motion of a drone based on nonlinear program (NLP) which optimizes the flight efficiency of the drone or visibility of the target with constraints such as safety. 
In this approach, the dynamics of the drone can be explicitly considered in a  transition model. Also, the field-of-view constraint corresponding to the state of the drone can be included  as in \cite{penin2018vision}. 
Nevertheless, the gradient-based optimization in NLP is subject to the numerical issues as there exist multiple conflicting objectives such as the travel efficiency, safety of the drone, visibility of the target, and the desired relative distance from the target.    

\textcolor{black}{The second group  \cite{jeon2002online,chen2016tracking,jeon2019online} decomposes a single non-convex optimization into multi-stage to resolve the numerical issues from the non-convexity of the objectives and constraints.} In these works, a preplanning step first computes a set of corridors which ensure the safety or visibility of the target, and then a smooth trajectory is computed inside the corridors. 
However, the chasing trajectory computed in the corridors might lose smoothness of  high-order derivatives. 
As noted in \cite{ding2019efficient}, this can be explained from the fact that the discrete search algorithm such as A* does not reflect the high-order derivatives in generating the corridors.

\begin{figure}[t]
    \centering
    \includegraphics[width=0.45\textwidth]{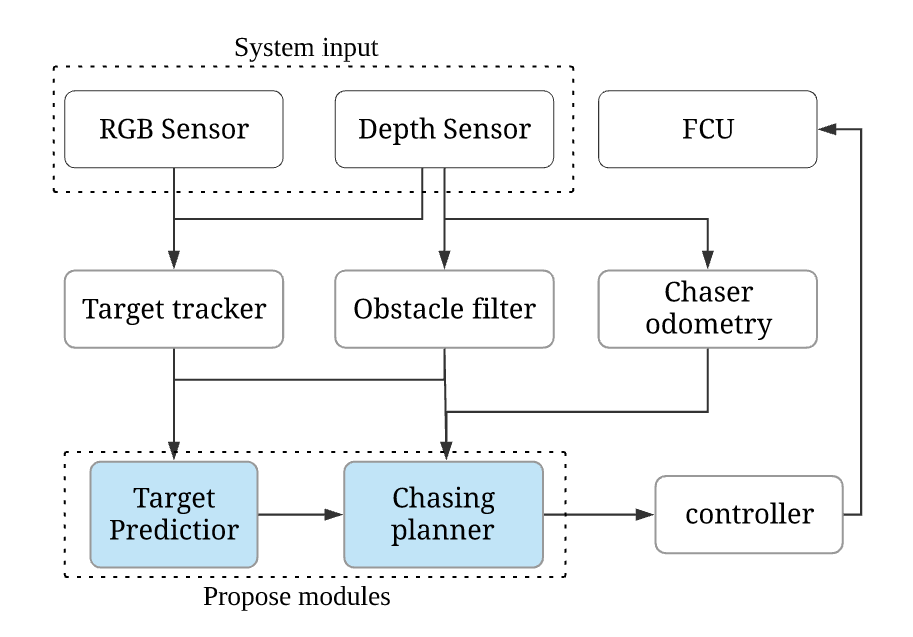}
    \caption{Overview of the proposed receding horizon chasing system implemented onboard. }
    \label{fig: onboard pipeline}
\end{figure}

\subsection{Sample-and-check strategy}
\label{sec: relwork sample and check }
To search for a direction to overcome the limitations mentioned in the previous section, 
we briefly shift our focus to recent works on a different problem: trajectory planning to reach a goal state \cite{mueller2015computationally,brescianini2018computationally}.
They plan a local trajectory toward a static goal by sampling a large set of motion primitives. Then, the feasibility of the individual candidate trajectory is checked in terms of constraints of interest, and the candidate attaining the best score is chosen as the final output. 
One of the key advantages of the methods is that the testing does not involve a loose approximation of the non-convex region. Moreover, the proposed test  can check a candidate primitive in an order of microseconds.    
Still, the feasibility checks only addressed the dynamic limits, and did not reflect other essential conditions such as collision-safety for more challenging applications.

\subsection{Contributions}
\label{subsec: contribution}
In the previous section, we highlighted the features of the sample-and-check approach \textcolor{black}{developed} to goal-reaching missions, where it enjoyed fast computation and the strict guarantee of the feasibility. To the best of our knowledge, such approach has not been applied to the chasing planner. 
Despite its capability to cope with the limitations mentioned in the related works on the chasing planner,  the following technical issues should be resolved to fully exploit the strengths: 1) efficient generation of the  high-quality candidate primitives, reflecting the movement of the target and changes of the environment, and 2) fast and sufficient checking for the primitives in terms of essential conditions in target-following missions.

This work presents an efficient chasing pipeline that addresses those issues with the following contributions:
\begin{itemize}
    \item An informed sampling method is proposed to generate candidate primitives for prediction and chasing.
    \item Fast feasibility checking methods are presented, which provide sufficient conditions for safety and visibility of a sampled primitive
    \item A systematic comparison is performed with multiple relevant works in various cases.
    \item An onboard integration including the target tracking, target prediction, and chasing planner is implemented and tested on an actual drone.  
    
\end{itemize}

The rest of this paper is structured as follows:
first, section \ref{sec: system overview} gives the overview of the proposed chasing system with the formulation of problems. 
Section \ref{sec_prediction} describes the method to calculate a reliable prediction  along with fast computation. 
In section \ref{sec_candid_traj}, we explain how to efficiently generate a set of high-quality candidate chasing trajectories in a closed form.
In section \ref{sec_feasi_check}, we present a strategy to check whether a given chasing candidate maintains visibility of the target against obstacles. 
The performance of the proposed pipeline is demonstrated  extensively in section \ref{sec_validation} with a comparative analysis, a high-fidelity simulation, and a real-world experiment.

\section{Overview}
\label{sec: system overview}

\subsection{{Pipeline}}
\label{subsec: system overview - pipeline}
The pipeline of the proposed system is visualized in Fig. \ref{fig: onboard pipeline}. The chasing drone is assumed to have a vision sensor suite consisting of an RGB camera and a LiDAR sensor, in order to detect the target and perform mapping obstacles.
\textcolor{black}{Utilizing them, this work focuses on the chasing pipeline where we first predict the target movement over a horizon so that the prediction does not intersect the obstacles. Then, the motion planner generates a chasing trajectory which ensures the safety of the drone and visibility of the predicted target against obstacles.
During the mission, we update the prediction and chasing motion when the prediction error exceeds a threshold or the planned trajectory becomes infeasible due to newly discovered obstacles. }

\subsection{Problem setup}
\label{sec: Problem setup}
Now, we state the problems to be solved by the chasing pipeline along with the assumptions. 
The maximum speed of the chaser is assumed to be faster than that of the target.
The future motion of the target is unknown to the chaser and should be predicted. We assume that the observation of the target's position by the sensor suite is given as Gaussian distribution. 
\textcolor{black}{The target is assumed to move smoothly minimizing high order derivatives of its state variables in a similar manner with \cite{chen2016tracking}.}
The obstacles in the environment can be either static or dynamic, both of which are modeled as ellipsoids of a constant shape whose center is a time-parameterized trajectory. 
Based on these assumptions, this section establishes two problems: 1) \textit{prediction problem}  and 2) \textit{chasing problem}. The former is addressed in Section \ref{sec_prediction}, and the latter in Section \ref{sec_candid_traj} and \ref{sec_feasi_check}.



\subsubsection{Prediction problem} 
The prediction problem focuses on forecasting the target's future motion that does not collide with obstacles  while maximizing the \textcolor{black}{likelihood of past observations}.
First, we represent the observations of the target position during the past $N_o$ time steps $t_n<0\;(1\leq n \leq N_o)$ as Gaussian distribution $\mathcal{N}(\targetMean_n,\targetVar_n)$ where $\targetMean_n \in \Rpos$ is the mean and $\targetVar_n\in\mathbb{R}^{3\times3}$ is the covariance.
Let us consider a set of $N_{obst}$ ellipsoidal obstacles with constant shape matrices $R_j\in \mathbb{R}^{3\times 3}$  and center trajectories  $\obstaclePosJ(t)\in\Rpos \; (t>0)$  where $1\leq j\leq N_{obst}$.
 Then, the safety of an arbitrary location $\pos \in\Rpos$ at $t$ with respect to the $j$th obstacle is ensured once the following condition is satisfied:
\begin{equation*}
    \lVert \pos - \obstaclePosJ(t) \rVert_{R_j}^{2} > 1  
\end{equation*}
where $\lVert \mathbf{x} \rVert_{R_j} = \sqrt{\mathbf{x}^T R_j \mathbf{x}}$ denotes the distance metric with respect to $R_j$.  
Based on the observation of the target and obstacles, the prediction trajectory $\targetPos(t) \in \Rpos$ over a future horizon $0< t \leq T$ is obtained by solving the problem below: 


\begin{problem}
(Prediction problem)
Given the observation $\mathcal{N}(\targetMean_n,\targetVar_n)$ of the dynamic target at the past time sequence $t_n <0 \;(n=1,2,\cdots,N_o)$ and the set of obstacles $(\obstaclePosJ(t),R_{j})_{j=1}^{N_{obst}}$ $(t>0)$,  the predicted trajectory of the target $\targetPos(t) $ for a future time window $t\in (0,T]$ is computed by solving the following:
\begin{equation*}
    \label{eqn_prediction_problem}
\begin{aligned}
& \underset{
\targetPos(t)}{\min}
&&\int_{0}^{T}\lVert{\hat{\mathbf{x}}}^{(2)}_q(t)\rVert^{2} dt+w_q \prod_{n=1}^{N_o}(p_{\mathcal{N}}({\hat{\mathbf{x}}_{q,n}}|\targetMean_n,\targetVar_n))^{-1}\\
& \text{{\normalfont subject to}} & & d_{q,j}(t) > 1 \; (\forall t \in (\textcolor{black}{0},T],\; \forall j \in \{1,..,N_{obst}\} ),
\end{aligned}
\end{equation*}
\end{problem}
\noindent where $d_{q,j}(t) = \lVert \targetPos(t)-\obstaclePosJ(t)\rVert_{R_j}^{2}$, ${\hat{\mathbf{x}}_{q,n}}=\targetPos(t_n)$, and $p_{\mathcal{N}}(\mathbf{x}|{\targetMean}_n,{\targetVar}_n)$ is the probability density function of  $\mathcal{N}({\targetMean}_n,{\targetVar}_n)$.

\subsubsection{Chasing problem}
Given the future prediction of the target $\targetPos(t)$ obtained by solving \textbf{Problem 1} and the obstacles $(\obstaclePosJ(t),R_j)_{j=1}^{N_{obst}}$, we establish the chasing problem to compute the motion of the chasing drone over the horizon $t\in(0,T]$ with respect to flat outputs $(\chaserPos(t),{\psi}_{p}(t))$ of the quadrotor dynamics where $\chaserPos(t)\in\Rpos$ is the position, and $\psi_{p}(t)$ is the yaw angle of the drone \cite{mellinger2011minimum}. 

\textcolor{black}{In this work, ${\psi}_{p}(t)$ is decided from the relative $x$ and $y$ coordinates between the target and chaser, so that the optical axis heads to $\targetPos(t)$ from $\chaserPos(t)$.}
In order to plan the translation $\chaserPos(t)$, we first formalize a set of objectives to be optimized, and feasibility conditions for the successful chasing mission. 

First of all, we penalize the integral of high-order derivatives (e.g. acceleration or jerk), which can be written as 
\begin{equation}
J_1 = \sum^{k_{max}}_{k=2} \int_{0}^{T} \rho_k \lVert \chaserPos^{(k)}(t) \rVert^2 dt,
\label{eqn_cost_smooth}
\end{equation}
where the set of derivative orders is chosen up to $k_{max}$. 
Next, we also consider the cost related with the safety from the $j$th obstacle from the below: 
\begin{equation}
    J_{2,\textcolor{blue}{j}} = \int_{0}^{T} c(\lVert \chaserPos(t) - \obstaclePosJ(t) \rVert_{R_j})dt,
\label{eqn_obstacle}
\end{equation}
where $c(l) \;(l\geq 0)$ is set as 
\begin{equation}
    c(l) = 
    \begin{cases}
    c_{min}+(c_{max}-c_{min}) \dfrac{(l-l_{s})^2}{l_{s}^2},  & 0 \leq l \leq l_{s} \\
    c_{min} & l > l_{s}\;.\\
    \end{cases}
    \label{eqn_cost_shaping}
\end{equation}
In \eqref{eqn_cost_shaping}, $c_{max}$, $c_{min}$ and $l_s$ are positive shaping parameters. 
In order to maintain the desired relative distance $l_{des}$ from the target, we include the following cost:

\begin{equation}
J_3 = \int_{0}^{T} {(\lVert \chaserPos(t) - \targetPos(t) \rVert - l_{des})}^2 dt.
\label{eqn_relative_dist}
\end{equation}

Also, we penalize high fluctuation of the heading of the drone to alleviate the motion blur.  For the purpose, we penalize the below: 
\begin{equation}
J_4 = \int_{0}^{T}  (\dot{\psi}_p(t))^{2} dt,
\label{yaw_rate}
\end{equation}
where $\dot{\psi}_p(t)$ is the rate of $\psi_p(t)$ expressed by the $x$ and $y$ components of  $\targetPos(t)$ and $ \chaserPos(t)$ as mentioned previously.

\begin{figure}[t] 
\centering
\includegraphics[width=0.4\textwidth]{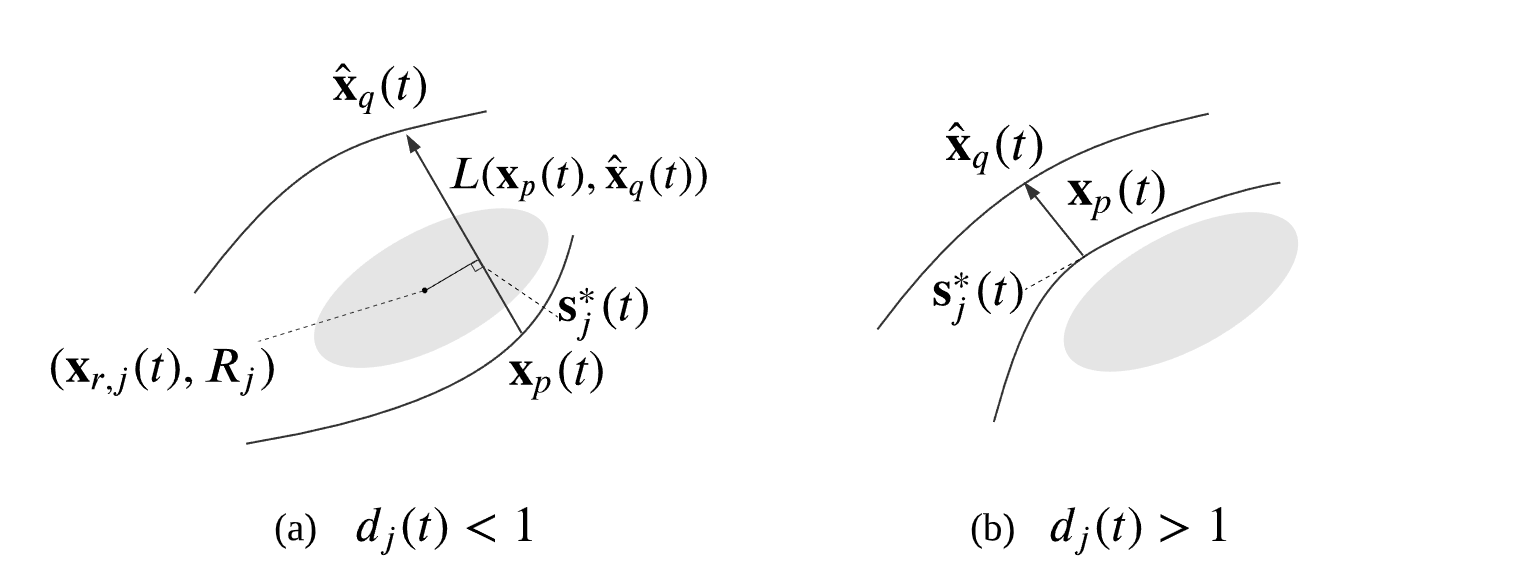}
\caption{Examples of an (a) infeasible  and (b) feasible case for chasing trajectory.
In (a),  the safety of the chaser is ensured, but the visibility is not satisfied, which determines it as an infeasible case.  In contrast,  $\chaserPos(t)$ in (b) satisfies all the feasibility conditions described as \eqref{eqn_feasibility}.
} 
\label{fig_feasibility}
\end{figure}

In addition to the objectives, the presence of obstacles requires two constraints to be satisfied by the chasing drone: 1) collision safety and 2) visibility of the target.  
For the chaser to be safe with respect to $j$th obstacle, ${\lVert \chaserPos(t)-\obstaclePosJ(t) \rVert}_{R_j}^{2}> 1$ should hold for the entire time window $(0,T]$. For the target to be visible to the chaser against obstacles, we should enforce  the line segment $L(\chaserPos(t),\targetPos(t))$ connecting the chaser $\chaserPos(t)$ and the target $\targetPos(t)$ not to intersect with any of the obstacles for the considered time window (see Fig. \ref{fig_feasibility}). To facilitate the analysis of the constraints, we introduce \textit{feasibility function} $d_j(t)$ as below: 
\begin{equation}
\begin{aligned}
         d_j(t) & = {\min}{\{\lVert \mathbf{s}(t)-\obstaclePosJ(t) \rVert}^{2}_{R_j} : \mathbf{s}(t)\in L(\chaserPos(t),\targetPos(t)) \}
\end{aligned}
\label{eqn_feasibility_function}
\end{equation}
Using the feasibility function, the conditions for safety and visibility concerning the $j$th obstacle can be encoded together as  
\begin{equation}
 d_j(t)  > 1  \; (\forall t \in (0,T]).   
 \label{eqn_feasibility}
\end{equation}
Putting the objectives and feasibility condition \eqref{eqn_cost_smooth}-\eqref{eqn_feasibility} together, we build the following motion planning problem for a drone to generate a chasing trajectory.

\begin{problem}
(Chasing problem)
Given the trajectories of the target $\targetPos(t)$ and the obstacles $(\obstaclePosJ(t),R_j)_{j=1}^{N_{obst}}$ for a horizon $(0,T]$, the chasing trajectory $\chaserPos(t)$ on the horizon is obtained by solving the following problem:
\begin{equation*}
\label{eqn_optimization}
\begin{aligned}
& \underset{
\chaserPos(t)}{\min}
&& w_{1}J_1 +w_{2}\sum_{j=1}^{N_{obst}}J_{2,j}+w_{3}J_3 + w_{4}J_4, \\ 
& \text{ {\normalfont subject to}} & &  \chaserPos^{(k)}(0) = \mathbf{x}_{p,0}^{(k)} \; (k=0, \cdots, 3) \\
& & &  d_j(t) > 1 \; (\forall t \in (0,T],\; \forall j \in \{1,..,N_{obst}\} ).
\end{aligned}
\end{equation*}
\end{problem}
\noindent where $\mathbf{x}_{p,0}^{(k)} \in \Rpos$ is the $k$th order derivative of the initial state of the chasing drone. \textbf{Problem 1} and \textbf{Problem 2} are optimization problems with multiple non-convex objectives and constraints, which render the difficulty in finding a global optimum if one relies on gradient-based optimization. To overcome the numerical issues and enjoy the advantages explained in section \ref{sec: relwork sample and check } and \ref{subsec: contribution}, \textcolor{black}{we aim to find an optimal solution of \textbf{Problem 1} and \textbf{Problem 2} based on the sample-and-check approach where we search a solution among a large set of primitives with feasibility testing.}

\section{forecasting the target trajectory}
\label{sec_prediction}
Based on the past observations  $\mathcal{N}(\targetMean_n,\targetVar_n)$ of the target from the sensor suite, we now present the method that generates prediction trajectory  to solve \textbf{Problem 1} based on the sample-and-check approach. 

\subsection{Generation of candidate prediction}
\label{subsec: Generation of candidate prediction}
Let us first consider the past time steps $\mathbf{t}_q = \{t_n |  n=1, \cdots N_o ,\; t_n<0\}$. For every time step $t_n \in \mathbf{t}_q$, we define a set consisting of $N_v$ points sampled on the contours of $\mathcal{N}(\targetMean_n,\targetVar_n)$ (see the small gray squares in Fig. \ref{fig_prediction}-(a) where $N_v = 5$). 
By selecting a point $\mathbf{v}_n \in \Rpos$ from the set for each $t_n$, we form their sequence  $\mathbf{V} = (\mathbf{v}_1,\mathbf{v}_2,...,\mathbf{v}_{N_o})$ with  $N_q = {(N_v)}^{N_o}$ possible arrangements.       
From each $\mathbf{V}$, we now derive a candidate prediction primitive $\targetPos(t)$ which efforts to track the sequence over $\mathbf{t}_q$, while optimizing the smoothness for the future horizon $0<t\leq T$. We compute $\targetPos(t)$ from the optimization below:   
\begin{equation}
\label{eqn_sampling_prediction}
 \targetPos(t) = \textrm{argmin}
 \int_{0}^{T} \lVert {\hat{\mathbf{x}}}^{(2)}_q(t) \rVert^2 dt + w \sum_{n=1}^{N_o}\lVert\targetPos(t_n) - \mathbf{v}_n \rVert^2, \\ 
\end{equation}
where $w$ is a positive importance weight. 

\begin{figure}[t] 
\centering
\includegraphics[width=0.35\textwidth]{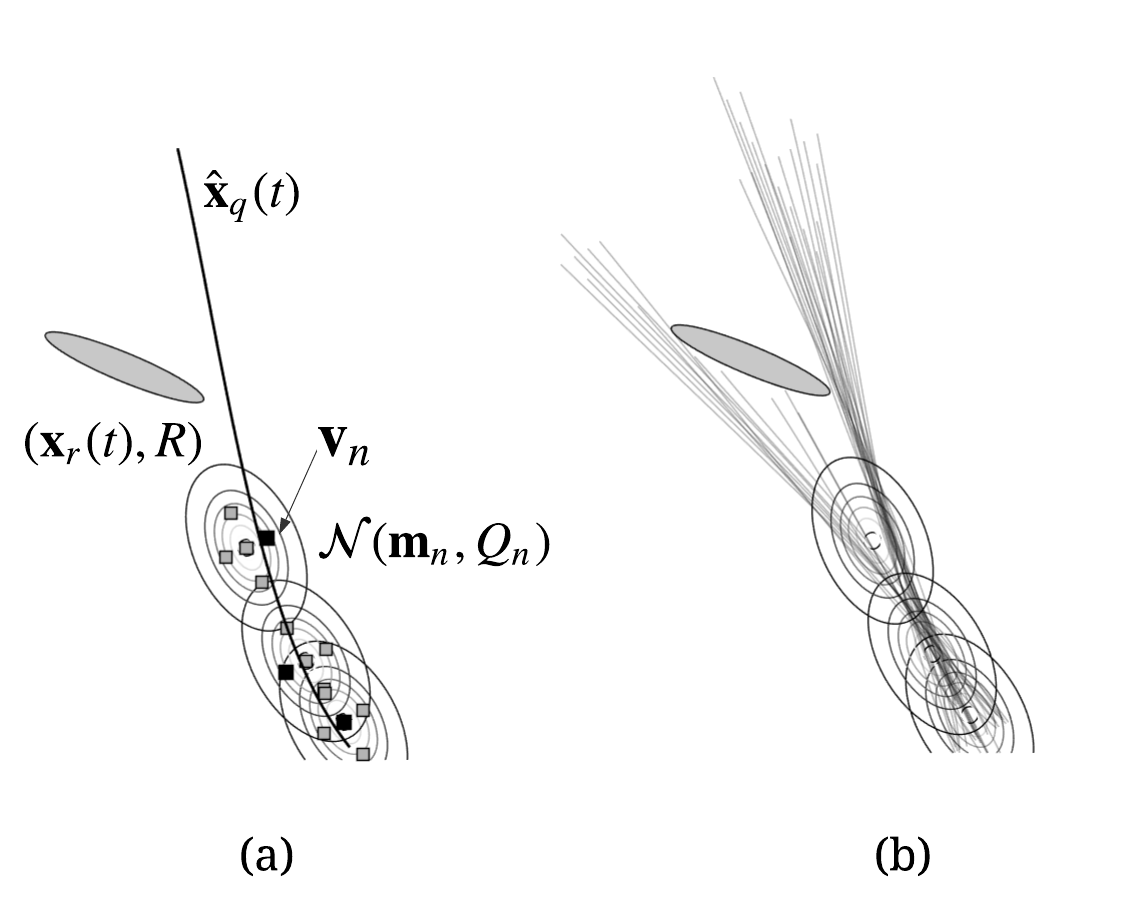}
\caption{(a) An example of a prediction primitive obtained from the sequence of points $\mathbf{V}$ (black squares) by permuting the points (gray squares) on the contours of $\mathcal{N}(\targetMean_n,\targetVar_n)$ at $t_n$. 
(b) The trajectories $\targetPos(t)$ among the set of obtained candidates, which does not collide with the obstacle.} 
\label{fig_prediction}
\end{figure}

Having the target prediction $\targetPos(t)$ is represented with a polynomial, we write  $\targetPos(t) = \mathbf{Q}^{\top}{\polyBase}_q(t)$ where  ${\polyBase}_q(t) = \begin{bmatrix}1 & t & t^2 & \cdots & t^{m_q}\end{bmatrix}^{\top}$ is a polynomial basis with degree $m_q$, and $\mathbf{Q} = \begin{bmatrix} \mathbf{q}_{_{(x)}} & \mathbf{q}_{_{(y)}} & \mathbf{q}_{_{(z)}} \end{bmatrix} \in \mathbb{R}^{(m_q+1)\times 3}$ is  coefficients for three-dimensional trajectories. Plugging them into \eqref{eqn_sampling_prediction} leads to the following unconstrained quadratic programming for each axis, \textcolor{black}{i.e.  $\mathbf{q} =  \mathbf{q}_{_{(x)}},\; \mathbf{q}_{_{(y)}}$, or  $\mathbf{q}_{_{(z)}} \in \mathbb{R}^{m_q+1} $:  } 

\begin{equation}
\label{eqn_prediction_optim_qp}
\begin{aligned}
 \underset{\mathbf{q}}{\min} \;  
\mathbf{q}^{\top}M_q\mathbf{q} + \mathbf{h}_q(\mathbf{V})^{\top} \mathbf{q}, 
\end{aligned}
\end{equation}
where 
\begin{gather*}
     {M}_q = \int_{0}^{T}\polyBase_q^{(2)}(t)(\polyBase_q^{(2)}(t))^{\top} dt + w \sum_{n=1}^{N_o} {\polyBase_q}(t_n){\polyBase_q}(t_n)^{\top}, \\
    \mathbf{h}_q(\mathbf{V}) = -2w\sum_{n=1}^{N_o}v_n\polyBase_q(t_n).
\end{gather*}
From the above equations, $v_n\in \R$ is the element of $\mathbf{v}_n$ corresponding to the considered axis. We can derive that $M_q$ is positive-definite for $N_o \geq 2$ and $m_q \geq  1$. In such setting, the element-wise coefficient $\mathbf{q}$ in  \eqref{eqn_prediction_optim_qp} is uniquely determined by:

\begin{equation}
    \label{eqn_prediction_inversion}
    \mathbf{q} = -\dfrac{1}{2}M_q^{-1}\mathbf{h}_q(\mathbf{V}).
\end{equation}
The example of a prediction candidate based on a sequence $\mathbf{V}$ can be found in Fig. \ref{fig_prediction}-(a). We now extend the above so that the coefficients of the whole set of $N_q $ possible candidate primitives can be computed simultaneously.
For the $i$th candidate ($1\leq i \leq N_q$), let us write the corresponding sequence as $\mathbf{V}_i $ and collect $\mathbf{h}_q(\mathbf{V}_i)$ into
\begin{equation*}
    H_q(\mathbf{V}_{1:N_q}) = \begin{bmatrix} \mathbf{h}_q(\mathbf{V}_1) & \mathbf{h}_q(\mathbf{V}_2) & \cdots & \mathbf{h}_q(\mathbf{V}_{N_q}) \end{bmatrix} \;.
\end{equation*}
Using this, the set of element-wise coefficients 
$\mathbf{q}_i  \; (1\leq i \leq N_q)$ can be computed by 
\begin{equation*}
\begin{bmatrix}\mathbf{q}_{1}&\mathbf{q}_{2}&\cdots&\mathbf{q}_{N_q}\end{bmatrix} = -\dfrac{1}{2}M_q^{-1}H_q(\mathbf{V}_{1:N_q}),
\end{equation*}
from which we can obtain all the candidate prediction trajectories for $\targetPos(t)$.   

\subsection{Feasibility checking and finalizing prediction}
We now focus on the feasibility check for the generated candidates with respect to the condition $d_q(t) = \lVert\targetPos(t)-\obstaclePos(t)\rVert_{R}^2>1 \; (0< t\leq T)$ in \textbf{Problem 1} (for brevity, the subscript $j$ for the obstacle index is omitted from now on). 
First, we can observe that $d_q(t)$ is a univariate polynomial on $t$ with degree $2\max\{m_q,m_r\}$ where $m_r$ is the order of considered polynomial for the obstacle center trajectory $\obstaclePos(t)$. 
In order to show $d_q(t)>1$ for $0< t\leq T$, it is sufficient to ensure the two conditions: 1) $d_q(0)>1$, and 2) $d_q(t)-1 $ has no real zeros on $0< t\leq T$. The former condition can be checked with a simple calculation of $d_q(0)$. The latter can be checked by a real-root counting algorithm for polynomials \cite{du2007amortized}. In this work, we use Sturm's theorem to count the number of distinct real roots, whose time complexity is $O(m^2)$ where $m$ is the degree of the polynomial. For a candidate prediction $\targetPos(t)$ with $d_q(0)>0$, we conclude it not to intersect with obstacles if  $d_q(t)-1=0$ does not have any real root. An example of candidate primitives which pass the tests are visualized in Fig. \ref{fig_prediction}-(b). 
Among the feasible candidates, we select the primitive which best minimizes the cost in \textbf{Problem 1}.

\section{Candidate chasing primitives generation}
\label{sec_candid_traj}

In the previous sections, we focused on the prediction of target's movement $\targetPos(t)$. Now, given  the prediction and the obstacle information $(\obstaclePos(t),R)$, \textbf{Problem 2} can be solved.

\subsection{Selection of view skeletons}
\label{subsec: Selection of view skeletons}
Given the predicted trajectory of the target $\targetPos(t) \; (0< t \leq T)$, we first consider a set of soft waypoints used as skeletons for the candidate motion primitives.  For a predicted position $\targetPos(t_n)$ at a time step $t_n \in \mathbf{t}_p = \{\dfrac{T}{N}n | n = 1 ,...,N\}$, we sample a set of $N_s$ points on a set of spheres centered at $\targetPos(t_n)$ having a different radius (see small black squares for $N_s=12$ in Fig. \ref{fig_candidate}-(a)). We call each point  \textit{view skeleton}.  We construct a sequence of points $\mathbf{U} = (\mathbf{u}_1,\mathbf{u}_2,...,\mathbf{u}_N)$ by choosing a skeleton $\mathbf{u}_n $ for each $n$, so that the number of possible sequences is  $N_p = (N_s)^{N}$. 

For each sequence $\mathbf{U}$, we now generate a candidate chasing polynomial $\chaserPos(t)$ which efforts to pass through view skeleton $\mathbf{u}_{n} $  at $t_n$ ($1\leq n \leq N$) and to minimize the high-order derivatives, while satisfying a set of initial conditions $\mathbf{x}_{p,0}^{(k)}\; (k = 0,1, ... ,3)$.
Reflecting them, we compute a candidate $\chaserPos(t)$ with the following equally constrained quadratic programming:
\begin{equation}
\label{eqn_sampling_optim}
\begin{aligned}
& \chaserPos(t) = \textrm{argmin}
&& \int_{0}^{T} \lVert {\mathbf{x}}^{(2)}_p(t) \rVert^2 dt + w_p \sum_{n=1}^{N}\lVert\chaserPos(t_n) - \mathbf{u}_n \rVert^2, \\ 
& \text{subject to} & &  \chaserPos^{(k)}(0) = \mathbf{x}_{p,0}^{(k)} \;\; (k = 0,1, ... ,3).
\end{aligned}
\end{equation}

\subsection{KKT condition for candidate polynomials}
\label{subsec: KKT condition for candidate polynomials}
In a similar manner with the derivation from \eqref{eqn_sampling_prediction} to \eqref{eqn_prediction_optim_qp}, we arrange \eqref{eqn_sampling_optim} with the coefficients of the primitives.
In order to represent the polynomial of a chasing trajectory, we define the basis of the polynomial $\mathbf{c}_p(t) = \begin{bmatrix} 1 & t & t^{2} & ... & t^{m_p} \end{bmatrix}^{\top}$, where $m_p$ is the order of polynomials for the chaser motion. Recalling the step taken in Section \ref{subsec: Generation of candidate prediction}, \eqref{eqn_sampling_optim} can be re-arranged into the following QP with an equality constraint on the per-axis polynomial coefficient $\mathbf{p}$ of $\chaserPos(t)$:
\begin{equation}
\label{eqn_sampling_optim_qp}
\begin{aligned}
& \min   
&& \mathbf{p}^{\top}M_p\mathbf{p} + \mathbf{h}_p(\mathbf{U})^T \mathbf{p}, \\
& \text{subject to} & &  A\mathbf{p} = \mathbf{b}
\end{aligned}
\end{equation}
where 
\begin{gather*}
    M_p = \int_{0}^{T}\polyBase_p^{(2)}(t)(\polyBase_p^{(2)}(t))^{\top} dt + w_p \sum_{n=1}^{N} {\polyBase_p}(t_n){\polyBase_p}(t_n)^{\top}, \\
    \mathbf{h}_p(\mathbf{U}) = -2w_p\sum_{n=1}^{N}u_n\polyBase_p(t_n), \\ 
    A = {\begin{bmatrix} \polyBase_p(0)& \polyBase_p^{(1)}(0) & ... & \polyBase_p^{(3)}(0) \end{bmatrix}}^{\top}, \\
    \mathbf{b} = {\begin{bmatrix} x_{p,0}& x^{(1)}_{p,0} & ... & x^{(3)}_{p,0}\end{bmatrix}}^{\top},
\end{gather*}
and $u_n\in \mathbb{R}$ is the element of a skeleton $\mathbf{u}_n$ along the considered axis. \textcolor{black}{Likewise, $x_{p,0}^{(k)}\;(k=0,\;\cdots,3)$ denotes the initial state of the axis.}  In  \eqref{eqn_sampling_optim_qp}, it can be derived that $A$ has row full rank if $m_p>3$ holds, and $M_p$ is a positive definite matrix if $N\geq2$ and $m_p\geq1$. 

With the parameter setting that satisfies these conditions, the following \textcolor{black}{KKT (Karush–Kuhn–Tucker)} condition yields a unique solution with an invertible KKT matrix $K$ \cite{boyd2004convex}:
\begin{equation*}
\underbrace{\begin{bmatrix}
2M_p & A^{\top} \\ A & 0 
\end{bmatrix}}_{K}
\underbrace{\begin{bmatrix}
\mathbf{p} \\ \mathbf{\lambda}
\end{bmatrix}}_{\mathbf{g}}
= 
\underbrace{\begin{bmatrix}
-\mathbf{h}_p(\mathbf{U}) \\ \mathbf{b}
\end{bmatrix}}_{\mathbf{y}},
\label{eqn_KKT}
\end{equation*}
where $\lambda$ is the associated Lagrange multiplier.
Let us denote the solution of the above as $\mathbf{g} = \begin{bmatrix}
\mathbf{p}^{\top} & \lambda^{\top} 
\end{bmatrix}^{\top}$, and the right-hand side as $\mathbf{y} \in \mathbb{R}^{(m_p+5)}$. 
Assigning an index $\mathbf{U}_i, \;1\leq i \leq N_p = (N_s)^{N}$, to each possible sequence, we let $\mathbf{y}(\mathbf{U}_i)$ represent $\mathbf{y}$ corresponding to each $\mathbf{U}_i$ and collect them into the following matrix:
\begin{equation*}
Y(\mathbf{U}_{1:N_p}) = \begin{bmatrix}
\mathbf{y}(\mathbf{U}_1) & \mathbf{y}(\mathbf{U}_1), & \cdots & \mathbf{y}(\mathbf{U}_{N_p})
\end{bmatrix}.
\end{equation*}
Then the following closed-form solution computes all the coefficients of candidate polynomials obtained from possible sequences of skeleton points:
\begin{equation}
\begin{bmatrix}
\mathbf{g}_1 & \mathbf{g}_2 & \cdots & \mathbf{g}_{N_p}   
\end{bmatrix}= K^{-1}Y(\mathbf{U}_{1:N_p})
\label{eqn_mat_multi}
\end{equation}
From $\mathbf{g}_i \; (1 \leq i \leq N_p)$, we can recover the polynomial coefficients of the whole set of candidate trajectories. \textcolor{black}{An example of the candidate chasing polynomials given view skeletons can be found in Fig. \ref{fig_candidate}-(b).}

\begin{figure}[t] 
\centering
\includegraphics[width=0.42\textwidth]{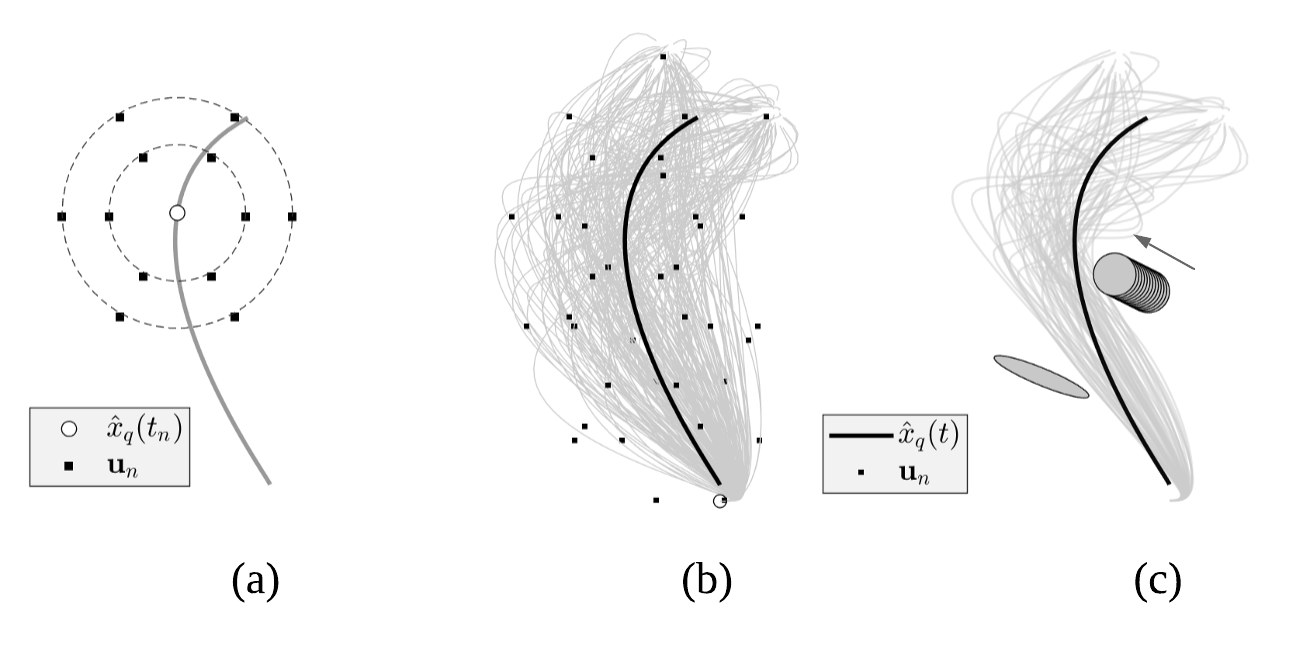}
\caption{(a) View skeletons $\mathbf{u}_n$ at a discrete target point $\targetPos(t_n)$. 
(b) The candidate polynomials (gray curves) around the target trajectory $\targetPos(t)$, which are based on the skeleton  $\mathbf{u}_n$.
(c) The set of primitives satisfying the proposed feasibility test. An obstacle is dynamic and moves along the direction of the arrow. } 
\label{fig_candidate}
\end{figure}

\section{Feasibility check for safety and visibility}
\label{sec_feasi_check}
In the previous section, we have described how to generate a candidate chasing trajectory around the predicted trajectory of the target. Now, we investigate whether each candidate polynomial $\chaserPos(t)$ satisfies safety and visibility for $0< t \leq T$ as formulated in the feasibility condition \eqref{eqn_feasibility}, given the polynomials of the target $\targetPos(t)$ and obstacles $(\obstaclePos(t),R)$ (subscript $j$ is omitted from \eqref{eqn_feasibility}).
In order to discuss the feasibility test, we derive a lower bound $d_{low}(t)$ of the feasibility function $d(t)$  in \eqref{eqn_feasibility_function}.  
The rationale of introducing $d_{low}(t) \leq d(t)$ is 1) to check the feasibility by simply counting the real roots of polynomial functions to be explained later and 2) to obtain sufficiency of the test. 

First, we explain the exact evaluation of  $d(t)$. 
A point $\mathbf{s}(t)$ on the line segment $L(\chaserPos(t),\targetPos(t))$ in \eqref{eqn_feasibility_function} can be expressed as $\mathbf{s}(t)  = (1-s)\chaserPos(t) +
s\targetPos(t)$ by introducing a scalar $s\in [0,1]$. 
Then,  \eqref{eqn_feasibility_function} can be rewritten as \eqref{eqn_feasibility_s} after defining a function $f(s,t) = \lVert (1-s)\chaserPos(t) +
s\targetPos(t) - \obstaclePos(t) \rVert^{2}_{R}$:
\begin{equation}
    d(t) = \underset{0 \leq s \leq 1}{\min} f(s,t),
    \label{eqn_feasibility_s}
\end{equation}
As the vector elements of $\chaserPos(t)$, $\targetPos(t)$ and $\obstaclePos(t)$ are all polynomials with respect to $t$,  we can observe that $f(s,t)$ is a bivariate polynomial whose degree is two with respect to $s$ and twice the largest order among $\chaserPos(t)$, $\targetPos(t)$ and $\obstaclePos(t)$. \textcolor{black}{In this work, we set the degree of $\chaserPos(t)$ to be the largest.} Noting this, $f(s,t)$  can be expressed as the following form: 
\begin{equation}
    f(s,t) = a_2(t)s^2+a_1(t)s+a_0(t),
    \label{eqn_f_quadratic}
\end{equation}
with the following coefficients
\begin{equation}
\begin{cases}
a_0(t) =  (\obstaclePos(t) - \chaserPos(t))^{\top}R(\obstaclePos(t) - \chaserPos(t)),\\
a_1(t) =  -2(\obstaclePos(t) - \chaserPos(t))^{\top}R(\targetPos(t)-\chaserPos(t)), \\ 
    a_2(t) = (\targetPos(t) - \chaserPos(t))^{\top}R(\targetPos(t) - \chaserPos(t)).
\end{cases}    
\label{eqn_coeff}
\end{equation}
As \eqref{eqn_f_quadratic} is a quadratic function on $s$ with a positive leading coefficient $a_2(t)$,  $d(t)$ can be determined analytically: 
\begin{equation}
    d(t) = 
    \begin{cases}
    a_0(t),  & \mathrm{if} \;\; \dfrac{a_1(t)}{-2a_2   (t)} < 0 \\
    -\frac{{(a_1(t))}^{2}}{4a_2(t)} +a_0(t), & \mathrm{if} \;\; 0 \leq \dfrac{a_1(t)}{-2a_2(t)} < 1\\
    a_2(t) + a_1(t) + a_0(t) & \mathrm{if} \;\; \dfrac{a_1(t)}{-2a_2(t)} \geq 1
    \end{cases}
    \label{eqn_exact_feasibility}
\end{equation}
The expression \eqref{eqn_exact_feasibility} with divided intervals is not convenient to check whether $d(t)>1$ for $0<t\leq T$.   
Rather than checking $d(t)>1$, we introduce a faster and sufficient test based on a lower bound of $d(t)$, by counting only the number of real roots of polynomials (more details on the advantages of using the lower bound is found in the next subsection).    
In order to derive the lower bound of the feasibility function $d(t)$ with a simpler form, we note the following property of a quadratic function: 

\begin{proposition}
\textit{For a univariate quadratic function $f(s) = a_2 s^2 + a_1 s + a_0$ with positive leading coefficient $a_2$, $\underset{0\leq s \leq 1}{\min} f(s) \geq \min \{d_p,d_q,d_s \}$ holds, where $d_p = a_0,d_q = a_0 + a_1 + a_2$ and $ d_s = a_1/2+a_0$}.
\label{lemma_quadratic}
\end{proposition}
\begin{proof}                              
Let us define two  functions that are tangent to $f(s)$ at $s=0$ and $s=1$ as $g_0(s)$ and $g_1(s)$. One can derive that $g_0$ and $g_1$ intersect at $s=1/2$ with $g_0(1/2) = g_1(1/2) =a_1/2+a_0 $. As $f(s)$ is a convex function with a positive leading coefficient, $f(s) \geq g_0(s)$ for $0\leq s \leq 1/2$ and $f(s) \geq g_1(s)$ for $1/2 \leq s \leq 1$. Thus, the minimum of $f(s)$ on $[0,1]$ is larger than the minimum of $g_0(s)$ on $\; 0\leq s \leq 1/2$ and minimum of $g_1(s)$ on $ \; 1/2\leq s \leq 1$, which are one of the three values $ \{ a_0,a_0 + a_1 + a_2,a_1/2+a_0 \} $. 
\end{proof}



\begin{proposition}
Let us define a function $d_{low}(t)$ by taking the minimum element from a set of three polynomials with respect to $t$ on an interval $[0,T]$:
\begin{equation*}
    d_{low}(t) =\min \{d_p(t),d_q(t),d_s(t)\},
\end{equation*}
where
\begin{gather*}
    d_p(t) = (\chaserPos(t) - \obstaclePos(t))^{\top}R(\chaserPos(t) - \obstaclePos(t)), \\
d_q(t) = (\targetPos(t) - \obstaclePos(t))^{\top}R(\targetPos(t) - \obstaclePos(t)), \\
d_s(t) = (\targetPos(t) - \obstaclePos(t))^{\top}R(\chaserPos(t) - \obstaclePos(t)), 
\end{gather*}
Then, for the feasibility function, $d(t)\geq d_{low}(t)$ holds for  $0\leq t \leq T$.
\label{proposition1}
\end{proposition}
\begin{proof}
According to \textbf{Proposition \ref{lemma_quadratic}}, it follows that $d(t) = \min_{0\leq s \leq 1} f(s,t) \geq \min\{d_p(t),d_q(t),d_s(t)\}$ where $d_p(t) = a_0(t),\; d_q(t) = a_0(t)+a_1(t)+a_2(t)$ and $d_s(t) = a_1(t)/2+a_0(t)$, when $f(s,t)$ is written as \eqref{eqn_f_quadratic}. Here, we can express  $d_p(t)$, $d_q(t)$ and $d_s(t)$ with $\chaserPos(t)$, $\targetPos(t)$, $\obstaclePos(t)$ and $R$ by plugging $a_0(t)$, $a_1(t)$ and $a_2(t)$ from \eqref{eqn_coeff}. 
          
\end{proof}
\textcolor{black}{Based on Proposition 2, we now state the following lemma where the lower-bound function $d_{low}(t)$ is used to test whether feasibility condition $d(t) > 1$ holds for  $\forall t \in (0,T]$. }
\begin{lemma}
Given $d_{low}(0) > 1$,  the safety of $\chaserPos(t)$ and visibility of $\targetPos(t)$ from $\chaserPos(t)$ are guaranteed for $\forall t \in (0,T]$ if none of three polynomial equations $d_q(t)=1$, $d_p(t)=1$ and $d_s(t)=1$ has a real root.  
\end{lemma}
\begin{proof}
The condition $d_{low}(0) > 1$ implies that the three polynomials $\{d_p(t),d_q(t),d_s(t)\}$ are larger than $1$ at $t = 0$. As the three polynomials are continuous, all of them remain larger than $1$ in $(0,T]$ if none of them intersects $1$ in the interval. Given this, $d_{low}(t) > 1$ holds for the entire duration $(0,T]$, leading to  $d(t)\geq d_{low}(t)>1$. Thus,  the safety and visibility are guaranteed from the definition of $d(t)$ in \eqref{eqn_feasibility_function}.
\end{proof}

According to \textbf{Lemma 1}, the feasibility for $\chaserPos(t)$ can be strictly guaranteed by checking non-existence of the real roots of all the three polynomial equations $d_q(t)=1$, $d_p(t)=1$ and $d_s(t)=1$, given the satisfied initial condition $d_{low}(0)>1$. Recalling that the prediction $\targetPos(t)$ computed in Section \ref{sec_prediction} has already ensured no real roots of $d_q(t)=1$ in the considered interval, 
only the remaining two equations need to be checked in a similar way with the feasibility check for prediction. An example of primitives which passed the proposed test is demonstrated in Fig. \ref{fig_candidate}-(c). 
Among the feasible candidate chasing primitives, we determine the final chasing trajectory as the one that best minimizes the cost in \textbf{Problem 2}.

\begin{figure}[t] 
\centering
\includegraphics[width=0.48\textwidth]{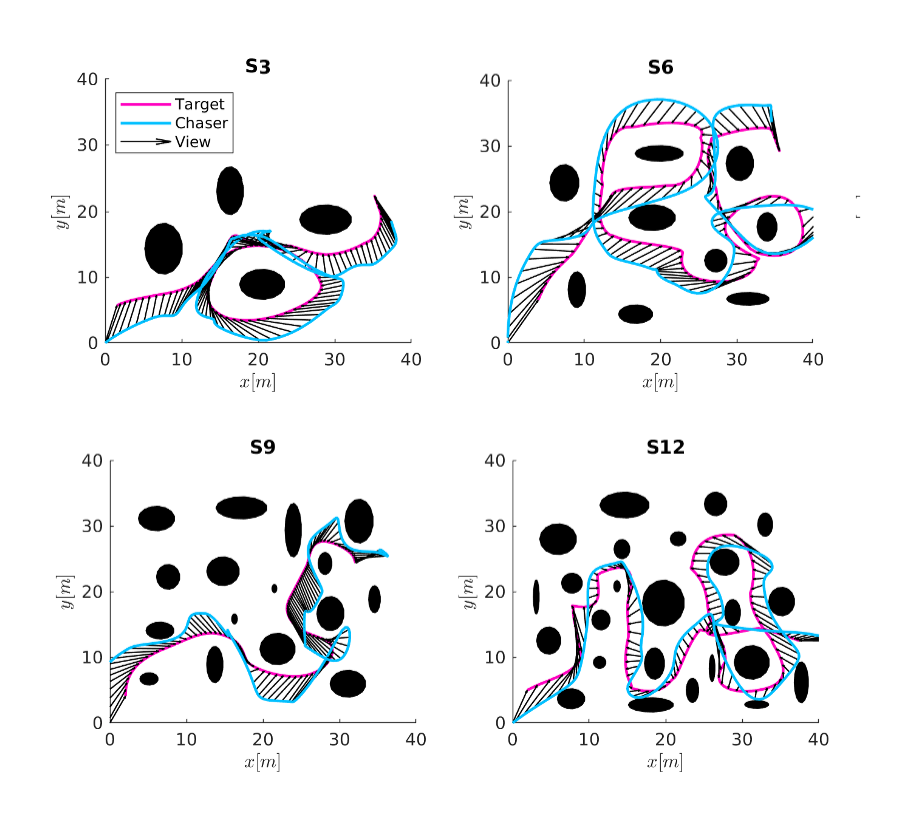}
\caption{Simulation results of S3, S6, S9, and S12 among tested 12 scenarios  (S1-S12). } 
\label{fig_comparison_jbs_only}
\end{figure}

\begin{figure}[t] 
\centering
\includegraphics[width=0.48\textwidth]{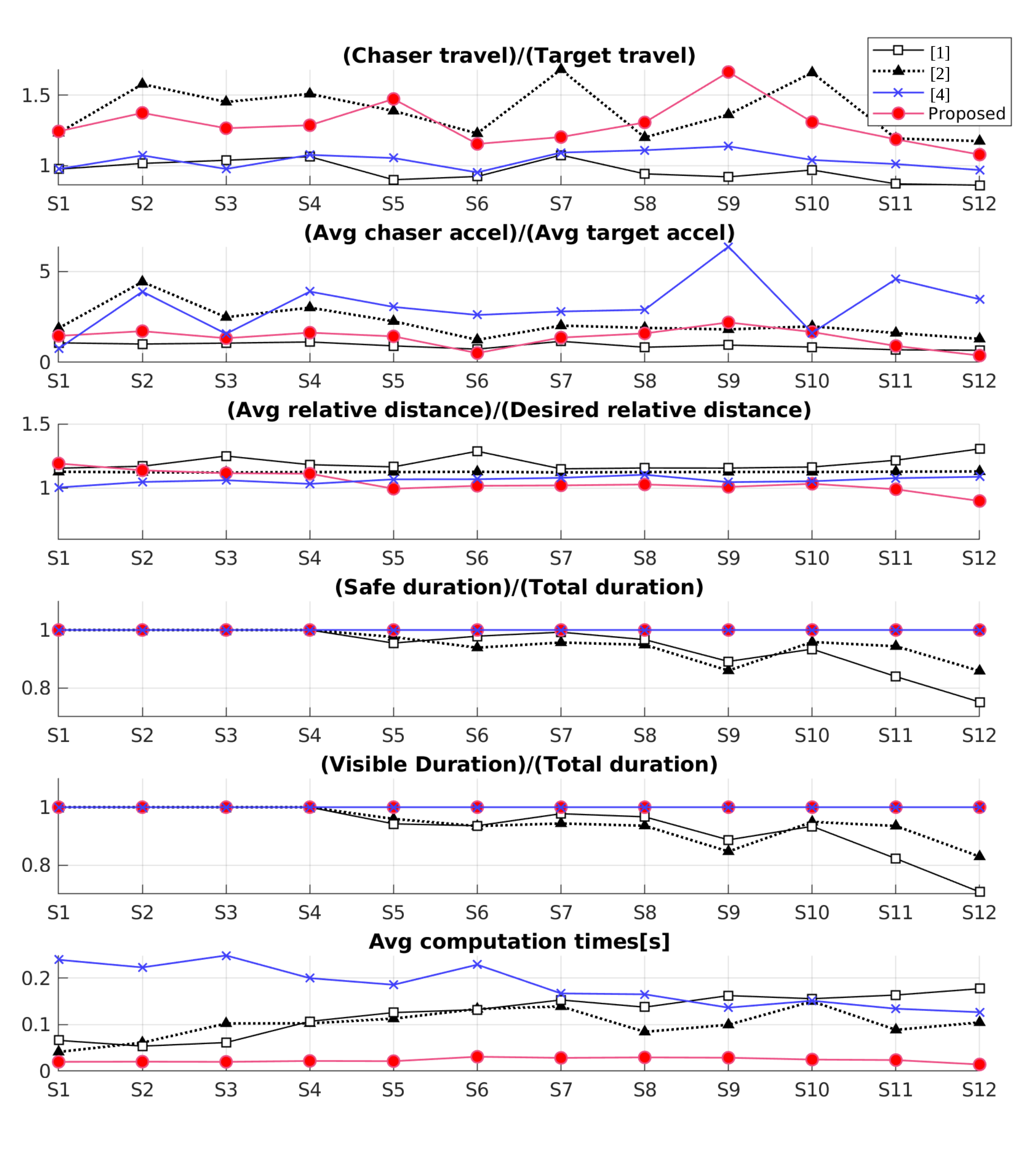}
\caption{Comparison result of the benchmark simulations S1-S12 with the state-of-the-art algorithms. \textcolor{black}{The x-axis denotes the index of the scenario.} } 
\label{fig_comparison_plot}
\end{figure}

\section{validations}
\label{sec_validation}

We now validate the presented chasing method in three different ways. 
First, we analyze the advantages of our algorithm from the comparison with three different RHPs including two non-convex optimization-based methods \cite{nageli2017real,penin2018vision} and a method using hierarchical planning structure \cite{jeon2002online}.
Second, the algorithm is validated using Airsim \cite{shah2018airsim} in high-fidelity game engine, where a virtual drone operates in a dense forest containing multiple dynamic obstacles and static obstacles.  
Third, an experiment is conducted by building a chasing drone with fully-onboard implementation of all the pipeline: odometry of the drone, mapping of the obstacles, localization and prediction of the dynamic target, chasing planner and flight control.     

\subsection{Comparison simulation}
To compare our chasing strategy with other works \cite{nageli2017real,penin2018vision,jeon2002online}, 
we consider 12 scenarios (S1-S12) built from 4 different maps and 3 variations of the planar target's movement for the duration of $50$ s in each map. The target's trajectory of S3, S6, S9, and S12 and obstacles are visualized in Fig. \ref{fig_comparison_jbs_only}.
Overall, the level of difficulty for chasing rises from S1 to S12, as the density of obstacles increases and the distance between the target and obstacles decreases. 
For all the four RHPs including ours, the planning horizon was set to $5$ s.
\textcolor{black}{ To focus on the planning aspect, in this simulation, the target's future trajectory over the horizon is given to the planners at each moment of replanning. All the obstacles are static and known a priori to the planners. The parameters for \cite{nageli2017real,penin2018vision} were selected as the best-performing one among a set of \textcolor{black}{$20$} pre-tuned values, and NLopt \cite{johnson2014nlopt} was utilized as the optimization solver. For our algorithm, in contrast, we used the same parameters for all 12 scenarios by using a candidate set of the $1,728$ quintic polynomial primitives ($N_s = 12$ and $N = 3$).
The simulations were performed on a standard i7 laptop computer with 16 GB RAM.  }

The comparison results of key performance measures in the 12 scenarios are summarized  in Fig. \ref{fig_comparison_plot}. 
As can be seen from the ratios of safe and visible duration, 
the planners based on the non-convex optimization failed to satisfy the essential conditions in  more difficult scenarios (S5-S12) where  the successful duration of \cite{penin2018vision} drops by almost $30$ percentage for S12. 
In contrast, our algorithm found the feasible chasing trajectories that
maintain the safety and visibility conditions in all cases even with the same parameter setting. 
The RHP of \cite{jeon2002online} also could ensure the feasibility conditions, by yielding a feasible corridor in the given resolution settings ($0.3$ m).  

In most cases, the hierarchical planner  performed best in terms of maintaining the desired relative distance, in that the measure was close to $1.00$ most closely.  
 However, the trajectory from \cite{jeon2002online} turns out to be the least smooth in general, when observed from acceleration costs (see second row of Fig. \ref{fig_comparison_plot}). \textcolor{black}{This is because the the trajectory is confined in the corridors which do not consider the smoothness} In contrast, the proposed performed the best for the acceleration cost.
 Regarding the average computation time, the proposed method consumed the lowest computation time in an order of $20 $ ms.

\subsection{High-fidelity simulation} \label{subsec:high-fidel}
Now, we test the proposed algorithm with the scenario where the environment is not known a priori and multiple dynamic obstacles are included in addition to a dense set of  static obstacles.
For the entire mission time of $53$ s in a forest of $100$ m $\times \;70$ m, the target moves at a constant speed of $3$ m/s, whose trajectory is marked by a red curve in  Fig. \ref{fig_unreal_result1}-(a). 
We used $4,096$ ($N_s = 8,\; N = 4$)  candidate chasing trajectories with a horizon of $2.5$ s. 
All the obstacles are discovered by a 16-channeled LiDAR sensor attached to the drone. 
The dynamic obstacles are represented with an elephant and a horse in the simulation, as visualized in Fig. \ref{fig_unreal_result1}-(b).

\begin{figure*}[t] 
\centering
\includegraphics[width=0.9\textwidth]{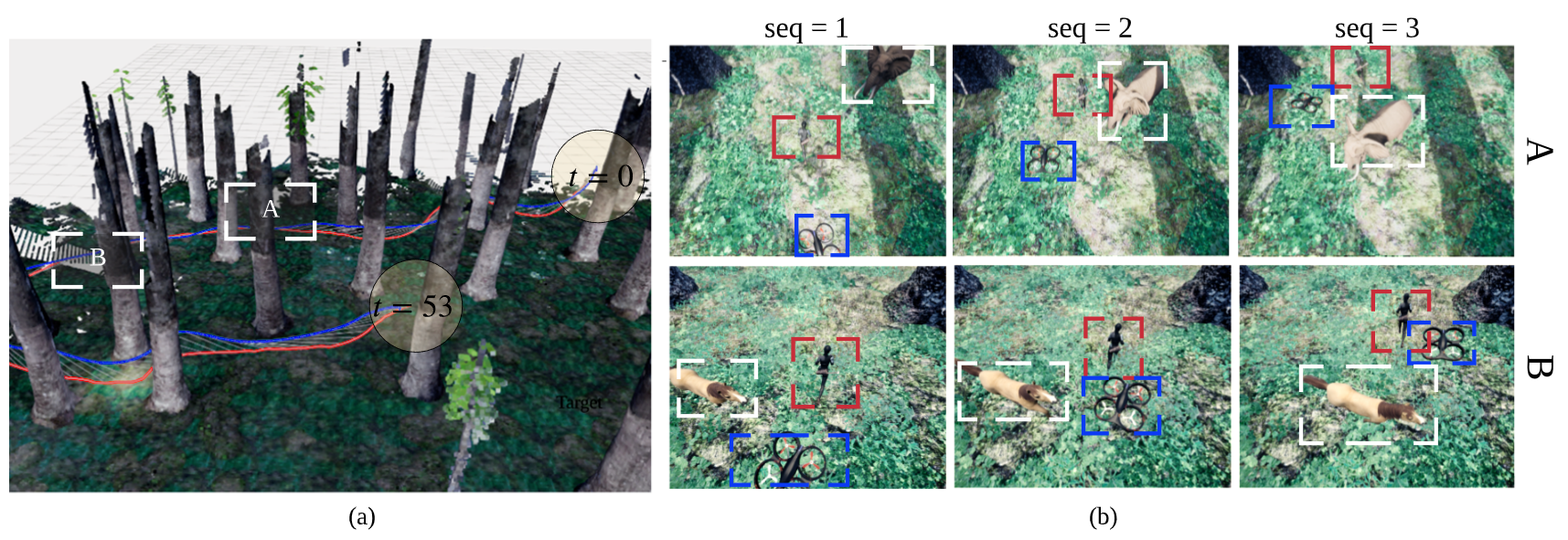}
\caption{(a) Perspective view of the position histories of the target (red) and the chaser (blue). The dynamic obstacles approach the chaser in the regions A and B. (b) Three snapshot sequence of the regions A and B  \textcolor{black}{over $2$ s (i.e. $1$ s intervals)} . The chaser and the target are denoted as the dashed boxes of the same color with (a). The dynamic obstacles (an elephant and a horse) are denoted in white boxes.    } 
\label{fig_unreal_result1} 
\end{figure*}

The trajectory of the drone can be found in Fig. \ref{fig_unreal_result1}-(a). The white-boxed regions A and B highlight the regions where the chaser encounters the dynamic obstacles.
The key snapshots taken from the indicated regions are in the \ref{fig_unreal_result1}-(b). There were two possible ways for the drone to avoid dynamic obstacles, for example, in A), avoiding to the right of the elephant or to the left. Although the former motion required less actuation than the latter, the chaser accelerated to take the latter for visibility of the target.    
The chasing planner ran at $25$ Hz in the simulation setting, despite the heavy computation load required to run the game engine in the same computer (Intel i7 and 32GB RAM). In the simulation, the distance ratio of the drone to the target was  $1.1$, showing the travel efficiency of our algorithm.
The results demonstrate that the motion of the drone planned by our method ensured the chasing objectives and feasibility conditions even in the challenging scenario which includes multiple high-speed dynamic obstacles and dense static obstacles.

\subsection{Real world experiment}
We now present the experimental result with the fully onboard integration of the proposed pipeline visualized in Fig. \ref{fig: onboard pipeline}. 
The integration on a chasing drone includes the visual odometry for the ego-motion, target perception with localization, obstacles mapping, target prediction and  motion planning. 
The experiment was performed in an indoor gym to test the real-world applicability of the proposed algorithm in  a dense obstacle setting (see Fig. \ref{fig_realworld_setting}-(a).  In an area of $8$ m $\times$ $15$ m, seven cylindrical obstacles of $56$ cm in diameter are set up.
For the target to be chased by the drone, we used a mobile robot with an orange cylinder as shown in Fig. \ref{fig_realworld_setting}-(b). The target was remotely operated by a human, and  its  future trajectory  is not known to the chaser.

An RGB monocular camera mvBlueFOX-MLC200 was used for the target perception. For  odometry, mapping and target localization, we used the Velodyne VLP16-Lite LiDAR.
For the flight controller, we used Pixhawk 4 FCU. Intel i7 NUC is mounted to the drone for onboard computation.   
For the chasing planner, we set the horizon as $2.5$ s and generated $1,296$ ($N_s = 6$ and $N = 4$)  candidate primitives. \textcolor{black}{The planned chasing motion is executed using the geometric controller \cite{lee2010geometric} and the state of the drone is fed from Lidar visual odomtery \cite{zhang2014loam}.}

\begin{figure}[t] 
\centering
\includegraphics[width=0.4\textwidth]{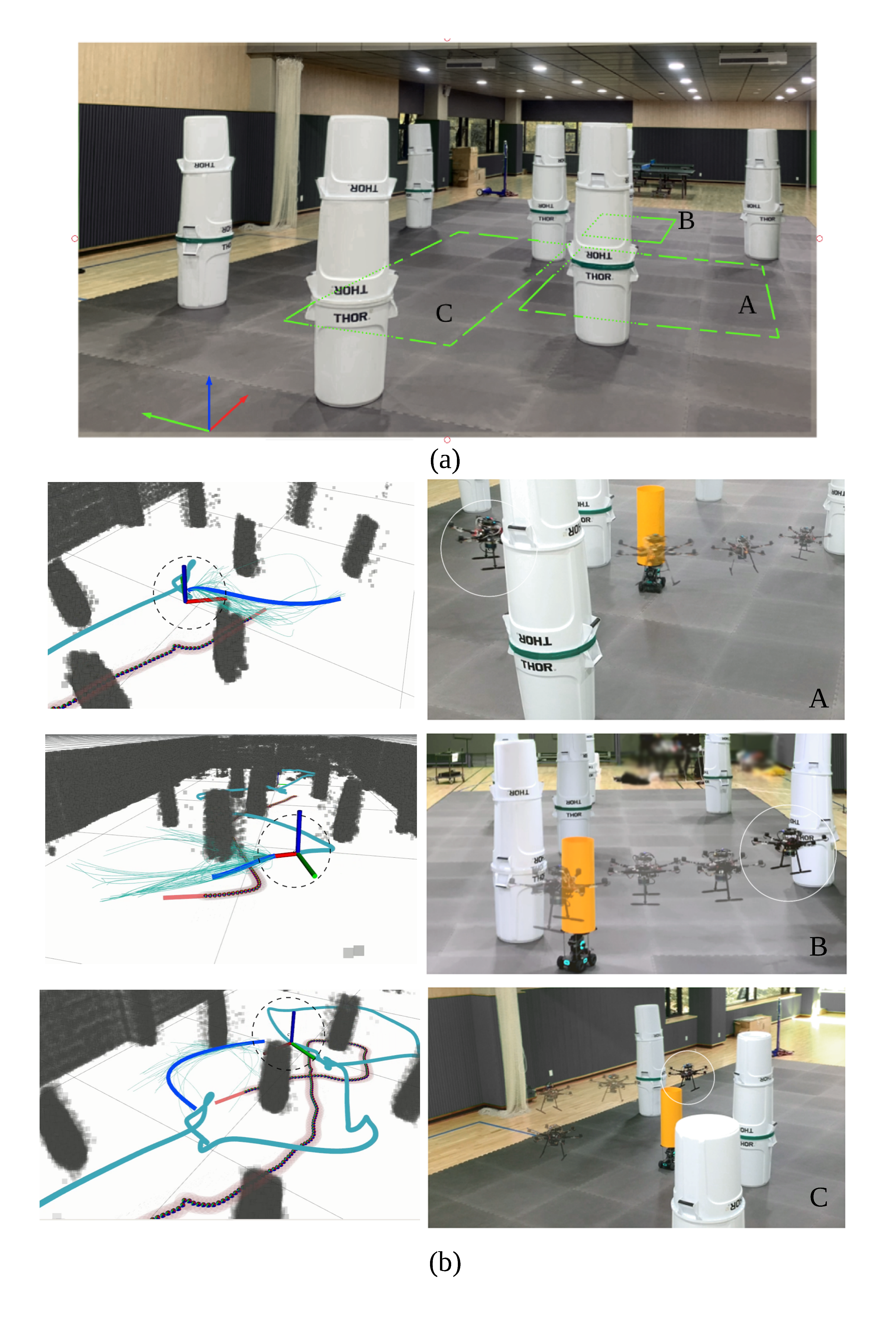}
\caption{(a) The environment for testing the proposed chasing pipeline. The target starts near the origin of the coordinate axis, and passes the green-dashed regions A, B, and C one by one. (b) Left: planning trajectory (blue) and target prediction (red). Right: corresponding snapshots of the left column.} 
\label{fig_realworld_setting}
\end{figure}

The snapshots taken in the scenes  A, B, and C in Fig. \ref{fig_realworld_setting}-(a) are shown in the right column of Fig. \ref{fig_realworld_setting}-(b) with the algorithm visualization in the left. 
As can be seen in the figure, the drone maintains its own safety and prevents the target-occlusion, while circumventing the obstacles.
The overall computation time for the prediction and planning was  $30$ ms on average, while  the entire pipeline such as the SLAM  and target detection algorithm ran on the same computer. The center of the detected target pixels on the image view is plotted in Fig. \ref{fig_realworld_chaserimage}, showing that the target is observed from the drone successfully during the entire mission. 

\begin{figure}[t] 
\centering
\includegraphics[width=0.4\textwidth]{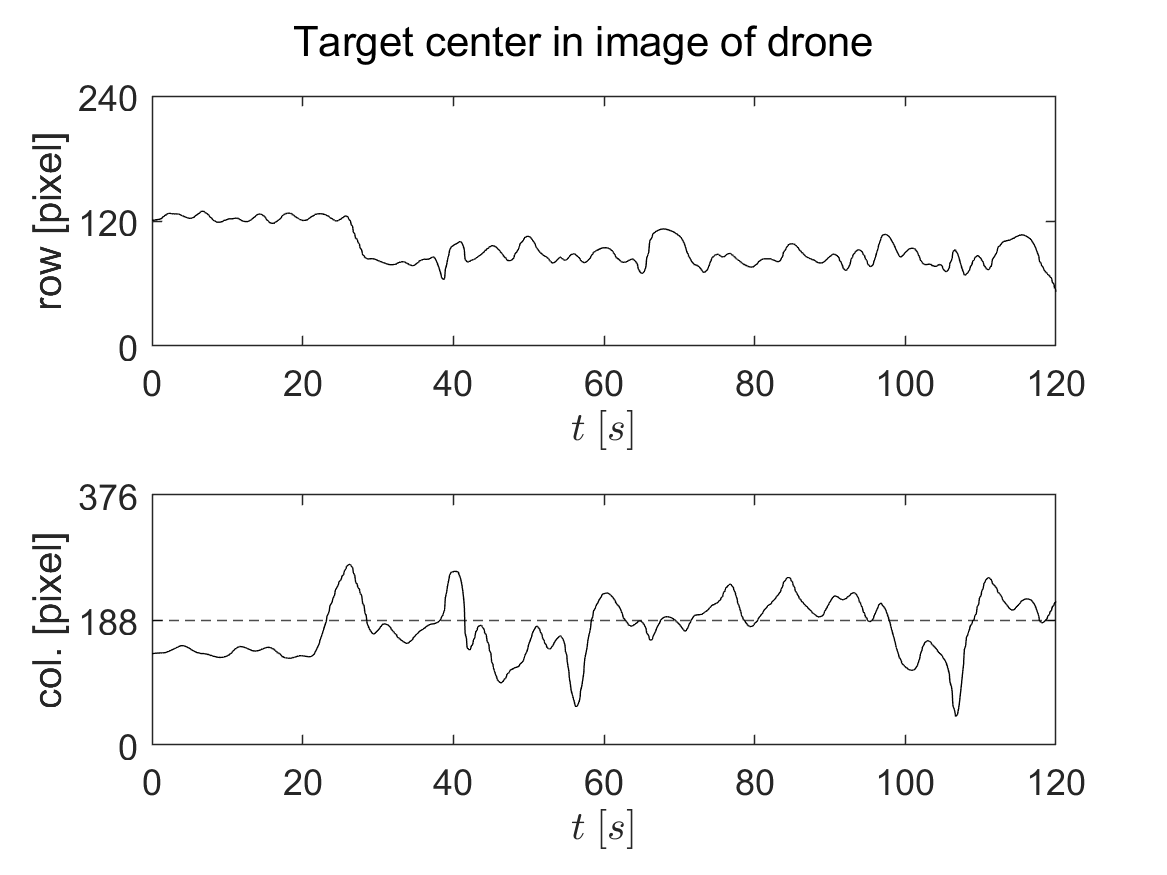}
\caption{The center of the detected target pixels from the image of the chaser drone. The size of the image is $376 \times 240$.  } 
\label{fig_realworld_chaserimage}
\end{figure}

\section{Conclusion}
This work proposed an efficient chasing system to tackle the difficult setting where the chaser has no information about the future target movement and obstacles. 
The system includes the target forecaster and the chasing planner, based on the sample-and-check approach introduced to overcome the drawback of the existing methods. 
To validate the performance of the pipeline, we conducted comparative simulation showing that the proposed approach  better ensured the essential conditions than the non-convex optimization method, and the trajectory of the chaser was smoother than the hierarchical method. Also, the computation time was the fastest among  the state-of-the-art algorithms. 
The entire pipeline was implemented on a real drone, showing the real-world applicability.
\textcolor{black}{For the future works, we will investigate an efficient method to adapt the number of candidate motion primitives and the configuration of the view skeleton on-the-fly, depending on the obstacles and target's movement. }


\bibliographystyle{ieeetr}
\bibliography{bibliography}
\biography{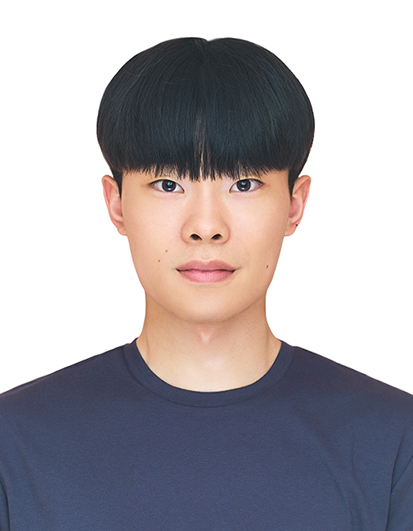}{Boseong Felipe Jeon}{received his B.S. degree in in Mechanical and Aerospace Engineering from Seoul National University in 2017. He is currently pursuing a Ph.D. degree in Aerospace Engineering at Seoul National University. His research interests include autonomous cinematography and aerial system.}
\biography{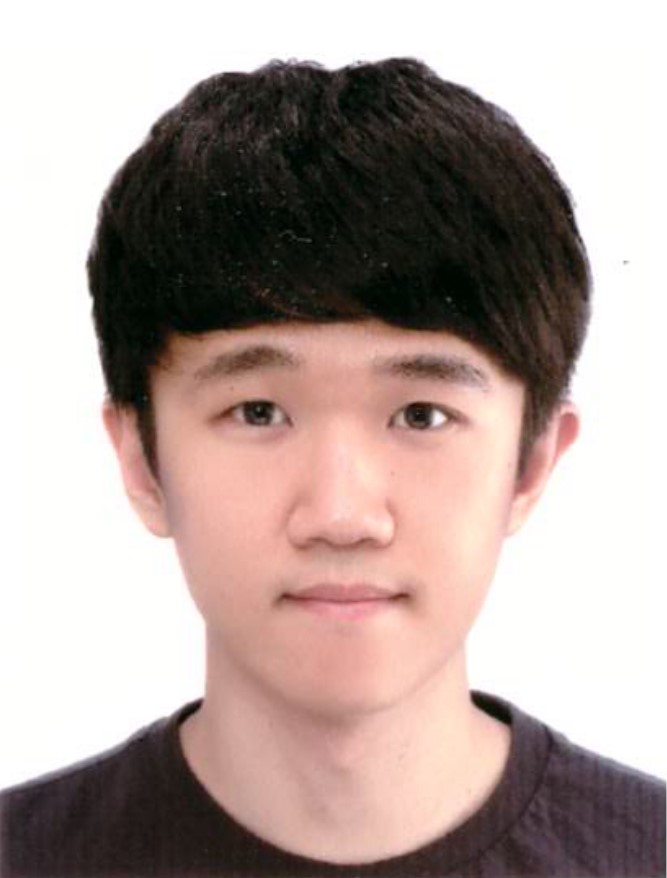}{Changhyeon Kim}{received the B.S. and M.S. degrees in Mechanical and Aerospace Engineering from Seoul National University in 2016 and 2018, respectively. He is currently pursuing a Ph.D. degree in Aerospace Engineering at Seoul National University. His research topics include 3-D reconstruction, visual navigation, camera-IMU-LiDAR fusion.}
\biography{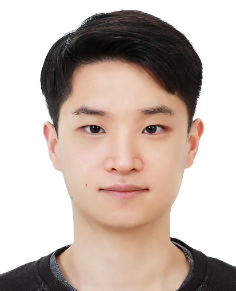}{Hojoon Shin}{is currently pursuing a B.S. degree in Mechanical and Aerospace Engineering at Seoul National University. His current research interests include multi-modal localization and mapping, collaborative navigation, and multi-agent systems.}
\biography{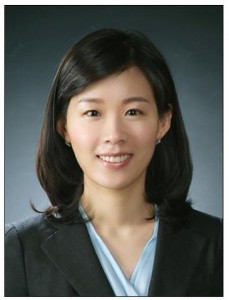}{H. Jin Kim}{received the B.S. degree
from Korea Advanced Institute of
Technology (KAIST) in 1995, and the
M.S. and Ph.D. degrees in Mechanical
Engineering from University of
California, Berkeley in 1999 and 2001,
respectively. From 2004-2009, she was an Assistant Professor in
the School of in Mechanical and Aerospace Engineering at
Seoul National University (SNU), Seoul, Korea, where she is
currently an Associate Professor. Her research interests include
applications of nonlinear control theory and artificial
intelligence for robotics, motion planning algorithm}

\clearafterbiography\relax

\end{document}